\documentclass[]{article}
\usepackage{proceed2e}
\usepackage{amssymb,amsmath,amsthm}
\usepackage{graphicx}
\usepackage{preamble}
\usepackage{natbib}
\usepackage{hyperref}
\usepackage{color}
\definecolor{mydarkblue}{rgb}{0,0.08,0.45}
\hypersetup{ %
    pdftitle={},
    pdfauthor={},
    pdfsubject={},
    pdfkeywords={},
    pdfborder=0 0 0,
    pdfpagemode=UseNone,
    colorlinks=true,
    linkcolor=mydarkblue,
    citecolor=mydarkblue,
    filecolor=mydarkblue,
    urlcolor=mydarkblue,
    pdfview=FitH}

\title{Optimally-Weighted Herding is Bayesian Quadrature}

\author{
 {\bf Ferenc Husz\'{a}r} \\
Department of Engineering\\
Cambridge University\\ 
\texttt{fh277@cam.ac.uk}
\And 
{\bf David Duvenaud } \\ 
Department of Engineering\\ 
Cambridge University \\
\texttt{dkd23@cam.ac.uk}
}

\begin{document} 
 
\maketitle 
 
\begin{abstract} 
Herding and kernel herding are deterministic methods of choosing samples which summarise a probability distribution.  A related task is choosing samples for estimating integrals using Bayesian quadrature.  We show that the criterion minimised when selecting samples in kernel herding is equivalent to the posterior variance in Bayesian quadrature.  We then show that sequential Bayesian quadrature can be viewed as a weighted version of kernel herding which achieves performance superior to any other weighted herding method. We demonstrate empirically a rate of convergence faster than $\mathcal{O}(1/N)$.  Our results also imply an upper bound on the empirical error of the Bayesian quadrature estimate.
\end{abstract} 

\begin{figure}[h]
\centering
\includegraphics[width=\columnwidth]{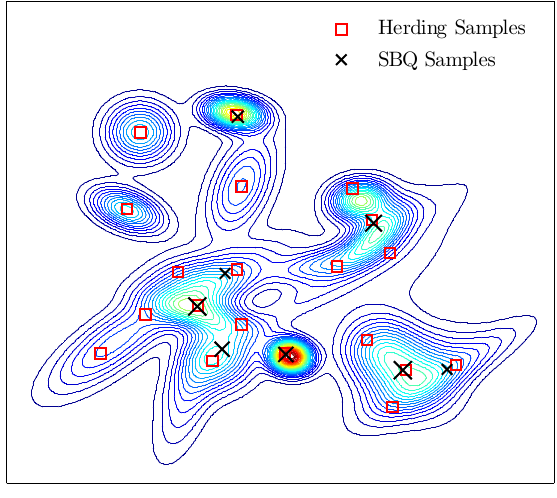}
\caption{The first 8 samples from sequential Bayesian quadrature, versus the first 20 samples from herding.  Only 8 weighted \sbq{} samples are needed to give an estimator with the same maximum mean discrepancy as using 20 herding samples with uniform weights.  Relative sizes of samples indicate their relative weights.}
\label{fig:fig1}
\end{figure}

\section{INTRODUCTION}
\paragraph{The problem: Integrals} A common problem in statistical machine learning is to compute expectations of functions over probability distributions of the form:
\begin{equation}
	Z_{f,p} = \int f(x) p(x) dx \label{eqn:integral}
\end{equation}
Examples include computing marginal distributions, making predictions marginalizing over parameters, or computing the Bayes risk in a decision problem. In this paper we assume that the distribution $p(x)$ is known in analytic form, and $f(x)$ can be evaluated at arbitrary locations.

Monte Carlo methods produce random samples from the distribution $p$ and then approximate the integral by taking the empirical mean $\hat{Z} = \frac{1}{N}\sum_{n=1}^{N}f_{x_n}$ of the function evaluated at those points. This non-deterministic estimate converges at a rate $\mathcal{O}(\frac{1}{\sqrt{N}})$. When exact sampling from $p$ is impossible or impractical, Markov chain Monte Carlo (MCMC) methods are often used. MCMC methods can be applied to almost any problem but convergence of the estimate depends on several factors and is hard to estimate \citep{CowlesCarlin96}. The focus of this paper is on quasi-Monte Carlo methods that -- instead of sampling randomly -- produce a set of pseudo-samples in a deterministic fashion. These methods operate by directly minimising some sort of discrepancy between the empirical distribution of pseudo-samples and the target distribution. Whenever these methods are applicable, they achieve convergence rates superior to the $\mathcal{O}(\frac{1}{\sqrt{N}})$ rate typical of random sampling.

In this paper we highlight and explore the connections between two deterministic sampling and integration methods: Bayesian quadrature (\bq{}) \citep{BZHermiteQuadrature,BZMonteCarlo} (also known as Bayesian Monte Carlo) and kernel herding \citep{chen2010super}. Bayesian quadrature estimates integral \eqref{eqn:integral} by inferring a posterior distribution over $f$ conditioned on the observed evaluations $f_{x_n}$, and then computing the posterior expectation of $Z_{f,p}$. The points where the function should be evaluated can be found via Bayesian experimental design, providing a deterministic procedure for selecting sample locations.

Herding, proposed recently by \cite{chen2010super}, produces pseudosamples by minimising the discrepancy of moments between the sample set and the target distribution. Similarly to traditional Monte Carlo, an estimate is formed by taking the empirical mean over samples $\hat{Z} = \frac{1}{N}\sum_{n=1}^{N}f_{x_n}$. Under certain assumptions, herding has provably fast, $\mathcal{O}(\frac{1}{N})$ convergence rates in the parametric case, and has demonstrated strong empirical performance in a variety of tasks.

\paragraph{Summary of contributions} In this paper, we make two main contributions.  First, we show that the Maximum Mean Discrepancy (MMD) criterion used to choose samples in kernel herding is identical to the expected error in the estimate of the integral $Z_{f,p}$ under a Gaussian process prior for $f$.  This expected error is the criterion being minimized when choosing samples for Bayesian quadrature.  Because Bayesian quadrature assigns different weights to each of the observed function values $f(\vx)$, we can view Bayesian quadrature as a weighted version of kernel herding.  We show that these weights are optimal in a minimax sense over all functions in the Hilbert space defined by our kernel.  This implies that Bayesian quadrature dominates uniformly-weighted kernel herding and other non-optimally weighted herding in rate of convergence.

Second, we show that minimising the MMD, when using \bq{} weights is closely related to the sparse dictionary selection problem studied in \citep{KrauseCevher10}, and therefore is approximately submodular with respect to the samples chosen. This allows us to reason about the performance of greedy forward selection algorithms for Bayesian Quadrature. We call this greedy method Sequential Bayesian Quadrature (\sbq{}).

We then demonstrate empirically the relative performance of herding, i.i.d random sampling, and \sbq{}, and demonstrate that \sbq{} attains a rate of convergence faster than $\mathcal{O}(1/N)$.

\section{HERDING} 

Herding was introduced by \cite{welling2009herding} as a method for generating pseudo-samples from a distribution in such a way that certain nonlinear moments of the sample set closely match those of the target distribution.  The empirical mean $\frac{1}{N}\sum_{n=1}^{N}f_{x_n}$ over these pseudosamples is then used to estimate integral \eqref{eqn:integral}.

\subsection{Maximum Mean Discrepancy}

For selecting pseudosamples, herding relies on an objective based on the maximum mean discrepancy \citep[MMD;\ ][]{Sriperumbudur2010}. MMD measures the divergence between two distributions, $p$ and $q$ with respect to a class of integrand functions $\mathcal{F}$ as follows:
\begin{align}
	\mmd_{\mathcal{F}}\left(p,q\right) = \sup_{f\in\mathcal{F}}\left\vert\int f_x p(x) dx - \int f_x q(x) dx \right\vert
\end{align}

Intuitively, if two distributions are close in the MMD sense, then no matter which function $f$ we choose from $\mathcal{F}$, the difference in its integral over $p$ or $q$ should be small. A particularly interesting case is when the function class $\mathcal{F}$ is functions of unit norm from a reproducing kernel Hilbert space (RKHS) $\He$. In this case, the MMD between two distributions can be conveniently expressed using expectations of the associated kernel $k(x, x')$ only \citep{Sriperumbudur2010}:
\begin{align}
MMD^2_{\He}(p,q) =& \sup_{\substack{f\in\He\\\Hnorm{f}=1}}\left\vert\int f_x p(x) dx - \int f_x q(x) dx\right\vert^2\label{eqn:rkhs-mmd}\\
	=& \Hnorm{\mu_{p} - \mu_{q}}^2\\
\nonumber	=&\iint k(x,y) p(x) p(y) dx dy\\
\nonumber	-2 &\iint k(x,y) p(x) q(y) dx dy\\
	+ &\iint k(x,y) q(x) q(y) dx dy,
\end{align}
where in the above formula $\mu_{p}=\int \phi(\vx)p(\vx)d\vx\in\He$ denotes the \emph{mean element} associated with the distribution $p$. For characteristic kernels, such as the Gaussian kernel, the mapping between a distribution and its mean element is bijective. As a consequence $\mmd_{\He}(p,q)=0$ if and only if $p=q$, making it a powerful measure of divergence.

Herding uses maximum mean discrepancy to evaluate of how well the sample set $\{\vx_1,\ldots,\vx_{N}\}$ represents the target distribution $p$:

\begin{align}
	\epsilon^2_{herding}&\left(\{\vx_1,\ldots,\vx_{N}\}\right) = \mmd^2_{\He}\left(p,\frac{1}{N}\sum_{n=1}^{N}\delta_{x_n}\right)\\
\nonumber	=&\iint k(x,y) p(x) p(y) dx dy\\
		- &\frac{1}{N}\sum_{n=1}^{N}\int k(x,x_n) p(x) dx
		+ \frac{1}{N^2}\sum_{n,m=1}^{N} k(x_n,x_m)
\label{eq:mmd_assumption}
\end{align}
The herding procedure greedily minimizes its objective $\epsilon_{herding}\left(\{\vx_1,\ldots,\vx_{N}\}\right)$ , adding pseudosamples $\vx_n$ one at a time. When selecting the $n+1$-st pseudosample:
\begin{align}
\vx_{n+1} &\leftarrow \argmin_{\vx \in \mathcal{X}} \label{eqn:herding_criterion} \epsilon_{herding}\left(\{\vx_1,\ldots,\vx_{n},\vx\}\right)\\
	&= \argmax_{\vx \in \mathcal{X}} 2 \expectargs{\vx' \sim p}{k(\vx, \vx')} - \frac{1}{n+1}\sum_{m=1}^{n} k(\vx,\vx_m)\mbox{,}\notag
\end{align}
assuming $k(\vx,\vx) = \mbox{const}$.
The formula \eqref{eqn:herding_criterion} admits an intuitive interpretation: the first term encourages sampling in areas with high mass under the target distribution $p(\vx)$.  The second term discourages sampling at points close to existing samples. 

Evaluating \eqref{eqn:herding_criterion} requires us to compute $\expectargs{\vx' \sim p}{k(\vx, \vx')} $, that is to integrate the kernel against the target distribution. Throughout the paper we will assume that these integrals can be computed in closed form. Whilst the integration can indeed be carried out analytically in several cases \citep{Song2008,chen2010super}, this requirement is the most pertinent limitation on applications of kernel herding, Bayesian quadrature and related algorithms.

\subsection{Complexity and Convergence Rates}

Criterion \eqref{eqn:herding_criterion} can be evaluated in only $\mathcal{O}(n)$ time. Adding these up for all subsequent samples, and assuming that optimisation in each step has $\mathcal{O}(1)$ complexity, producing $N$ pseudosamples via kernel herding costs $\mathcal{O}(N^2)$ operations in total.

In finite dimensional Hilbert spaces, the herding algorithm has been shown to reduce $\mmd$ at a rate $\mathcal{O}(\frac{1}{N})$, which compares favourably with the $\mathcal{O}(\frac{1}{\sqrt{N}})$ rate obtained by non-deterministic Monte Carlo samplers. However, as pointed out by \cite{bach2012equivalence}, this fast convergence is not guaranteed in infinite dimensional Hilbert spaces, such as the RKHS corresponding to the Gaussian kernel.

\section{BAYESIAN QUADRATURE} 

\begin{figure}
\centering
\includegraphics[width=\columnwidth]{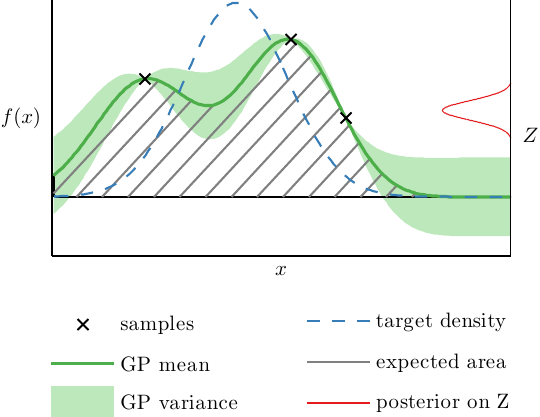}
\caption{An illustration of Bayesian Quadrature.  The function $f(x)$ is sampled at a set of input locations.  This induces a Gaussian process posterior distribution on $f$, which is integrated in closed form against the target density, $p(\vx)$.  Since the amount of volume under $f$ is uncertain, this gives rise to a (Gaussian) posterior distribution over $Z_{f,p}$.}
\label{fig:bq_intro}
\end{figure}

So far, we have only considered integration methods in which the integral \eqref{eqn:integral} is approximated by the empirical mean of the function evaluated at some set of samples, or pseudo-samples.  Equivalently, we can say that Monte Carlo and herding both assign an equal $\frac{1}{N}$ weight to each of the samples.

In \citep{BZMonteCarlo}, an alternate method is proposed: Bayesian Monte Carlo, or Bayesian quadrature (\bq).  \bq{} puts a prior distribution on $\vf$, then estimates integral \eqref{eqn:integral} by inferring a posterior distribution over the function $\vf$, conditioned on the observations $\vf(\vx_n)$ at some query points $\vx_n$.  The posterior distribution over $f$ then implies a distribution over $Z_{f,p}$.  This method allows us to choose sample locations $\vx_n$ in any desired manner. See Figure \ref{fig:bq_intro} for an illustration of Bayesian Quadrature.


\subsection{ BQ Estimator}

Here we derive the \bq{} estimate of \eqref{eqn:integral}, after conditioning on function evaluations $\vf(\vx_1) \dots \vf(\vx_N)$, denoted as $f(\vX)$.  The Bayesian solution implies a distribution over $Z_{f,p}$.  The mean of this distribution, $\expectargs{}{Z}$ is the optimal Bayesian estimator for a squared loss.

For simplicity, $\vf$ is assigned a Gaussian process prior with kernel function $k$ and mean $0$.  This assumption is very similar to the one made by kernel herding in Eqn.\ \eqref{eq:mmd_assumption}.

After conditioning on $\vf_{\vx}$, we obtain a closed-form posterior over $\vf$:
\begin{align}
p(\vf(\vx\st)|\vf(\vX)) = \N{\vf_{\vx\st}}{\mf(\vx\st)}{\cov(\vx\st,\vx\st')}
\end{align} 
where
\begin{align}
\mf(\vx\st) = & k(\vx\st, \vX) K^{-1} \vf(\vX) \\
\cov(\vx\st, \vx\st') = & k(\vx\st,\vx'\st) - k(\vx\st, \vX) K^{-1} k(\vX, \vx'\st)
\end{align} 
and $K = k(\vX, \vX)$. 
Conveniently, the \gp{} posterior allows us to compute the expectation of \eqref{eqn:integral} in closed form: 
%
\begin{align}
\expectargs{\gp}{Z} & = \expectargs{\gp}{\int f(\vx)p(\vx)d\vx}\\
 & = \int\!\!\! \int\!\! f(\vx) p(f(\vx)|\vf(\vX)) p(\vx) d\vx df\\
 & = \int\!\!\! \mf(\vx) p(\vx) d\vx \\
 & = \left[ \int\!\! k(\vx, \vX) p(\vx) d\vx \right] K^{-1} \vf(\vX) \\
 & = \vz^T K^{-1} \vf(\vX)
\label{eq:marg_mean_symbolic}
\end{align} 
where
\begin{align}
z_n & = \int\!\! k(\vx, \vx_n) p(\vx) d\vx = \expectargs{\vx' \sim p}{k(\vx_n, \vx')}.
\end{align}
Conveniently, as in kernel herding, the desired expectation of $Z_{f,p}$ is simply a linear combination of observed function values $\vf(\vx)$:
\begin{align}
\expectargs{\gp}{Z} & = \vz^T K^{-1} \vf(\vX) \\
    & = \sum_n w_{\bq}^{(n)} \vf(\vx_n)
\end{align}  
where
\begin{align}  
w_{\bq}^{(n)} & = \sum_m \vz_m^T K^{-1}_{mn}
\label{eq:bq_weights}
\end{align}
Thus, we can view the BQ estimate as a weighted version of the herding estimate.  Interestingly, the weights $\vw_{\bq}$ do not need to sum to 1, and are not even necessarily positive.

\subsubsection{Non-normalized and Negative Weights}

\begin{figure}
\centering
\includegraphics[width=\columnwidth]{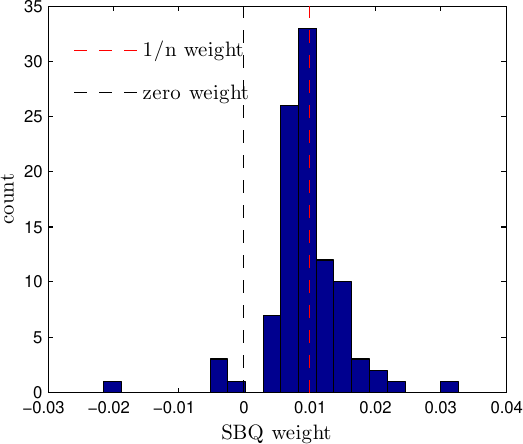}
\caption{A set of optimal weights given by \bq{}, after 100 \sbq{} samples were selected on the distribution shown in Figure \ref{fig:fig1}.  Note that the optimal weights are spread away from the uniform weight ($\frac{1}{N}$), and that some weights are even negative.  The sum of these weights is 0.93.}
\label{fig:weights100}
\end{figure}

When weighting samples, it is often assumed, or enforced \citep[as in][]{bach2012equivalence,Song2008}, that the weights $\vw$ form a probability distribution.  However, there is no technical reason for this requirement, and in fact, the optimal weights do not have this property.  Figure \ref{fig:weights100} shows a representative set of 100 \bq{} weights chosen on samples representing the distribution in figure \ref{fig:fig1}.  There are several negative weights, and the sum of all weights is 0.93.

Figure \ref{fig:weights_shrinkage} demonstrates that, in general, the sum of the Bayesian weights exhibits shrinkage when the number of samples is small.

\begin{figure}
\centering
\includegraphics[width=\columnwidth]{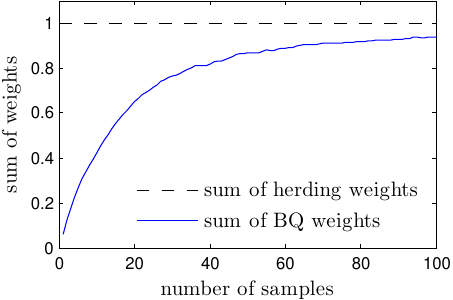}
\caption{An example of Bayesian shrinkage in the sample weights.  In this example, the kernel width is approximately $\nicefrac{1}{20}$ the width of the distribution being considered.  Because the prior over functions is zero mean, in the small sample case the weights are shrunk towards zero.  The weights given by simple Monte Carlo and herding do not exhibit shrinkage. }
\label{fig:weights_shrinkage}
\end{figure}



\subsection{Optimal sampling for BQ}

Bayesian quadrature provides not only a mean estimate of $Z_{f,p}$, but a full Gaussian posterior distribution. The variance of this distribution $\varianceargs{}{Z_{f,p}|f_{x_1}, \dots, f_{x_N}}$ quantifies our uncertainty in the estimate. When selecting locations to evaluate the function $f$, minimising the posterior variance is a sensible strategy. Below, we give a closed form formula for the posterior variance of $Z_{f,p}$, conditioned on the observations $f_{x_1} \dots f_{x_N}$, which we will denote by $\epsilon^2_{\bq{}}$.  For a longer derivation, see \cite{BZMonteCarlo}.
\begin{align}
\epsilon^{2}_{\bq{}}(\vx_1,\ldots,\vx_N) & = 
\varianceargs{}{Z_{f,p}|f_{x_1}, \dots, f_{x_N}} \\
& = \expectargs{\vx, \vx' \sim p}{k(\vx, \vx')} - \vz^T K^{-1} \vz\mbox{,}
\label{eq:marg_var_symbolic}
\end{align}
where $\vz_n = \expectargs{\vx' \sim p}{k(\vx_n, \vx')}$ as before. Perhaps surprisingly, the posterior variance of $Z_{f,p}$ does not depend on the observed function values, only on the location $x_n$ of samples. A similar independence is observed in other optimal experimental design problems involving Gaussian processes \citep{guestrin1}. This allows the optimal samples to be computed ahead of time, before observing any values of $f$ at all \citep{minka2000dqr}.

We can contrast the \bq{} objective $\epsilon^{2}_{\bq{}}$ in \eqref{eq:marg_var_symbolic} to the objective being minimized in herding, $\epsilon^{2}_{herding}$ of equation \eqref{eq:mmd_assumption}. Just like $\epsilon^{2}_{herding}$, $\epsilon^{2}_{\bq{}}$ expresses a trade-off between accuracy and diversity of samples. On the one hand, as samples get close to high density regions under $p$, the values in $\vz$ increase, which results in decreasing variance. On the other hand, as samples get closer to each other, eigenvalues of $K$ increase, resulting in an increase in variance. 

In a similar fashion to herding, we may use a greedy method to minimise $\epsilon^{2}_{\bq{}}$, adding one sample at a time. We will call this algorithm \emph{Sequential Bayesian Quadrature} (\sbq{}):
\begin{align}
\vx_{n+1} &\leftarrow \argmin_{\vx \in \mathcal{X}} \epsilon_{\bq{}}\left(\{\vx_1,\ldots,\vx_{n},\vx\}\right)
\end{align}
Using incremental updates to the Cholesky factor, the criterion can be evaluated in $\mathcal{O}(n^2)$ time. Iteratively selecting $N$ samples thus takes $\mathcal{O}(N^3)$ time, assuming optimisation can be done on $\mathcal{O}(1)$ time.

\section{RELATING $\varianceargs{}{Z_{f,p}}$ TO $\mmd$}

The similarity in the behaviour of $\epsilon^{2}_{herding}$ and $\epsilon^{2}_{\bq{}}$ is not a coincidence, the two quantities are closely related to each other, and to \mmd.
	
\begin{prop} The expected variance in the Bayesian quadrature $\epsilon^{2}_{\bq{}}$  is the maximum mean discrepancy between the target distribution $p$ and $q_{\bq{}}(x) = \sum_{n=1}^{N}w^{(n)}_{\bq{}}\delta_{x_n}(x)$
\end{prop}
\begin{proof}
The proof involves invoking the representer theorem, using bilinearity of scalar products and the fact that if $f$ is a standard Gaussian process then $\forall g\in\He: \left\langle f,g\right\rangle \sim \mathcal{N}(0,\Hnorm{g}^2)$:
\begin{align}
&\varianceargs{}{Z_{f,p}\vert f_{x_1}, \dots, f_{x_N}}=\\
	&= \mathbb{E}_{f\sim GP} \left( \int f(x) p(x) dx - \sum_{n=1}^{N}w^{(n)}_{\bq{}} f(x_n)\right)^2\\
	&= \mathbb{E}_{f\sim GP} \left( \int \left\langle f, \phi (x)\right\rangle p(x) dx - \sum_{n=1}^{N}w^{(n)}_{\bq{}} \left\langle f, \phi (x_n)\right\rangle\right)^2\\
	&= \mathbb{E}_{f\sim GP} \left\langle f ,  \int\phi(x) p(x) dx - \sum_{n=1}^{N}w^{(n)}_{\bq{}}\phi(x_n)\right\rangle^2\\
	&= \Hnorm{\mu_p - \mu_{q_{\bq{}}}}^2\\
	&= \mmd^2(p,q_{\bq{}})
\end{align}
\end{proof}

We know that the the posterior mean $\expectargs{\gp}{Z_{f,p}\vert f_1,\ldots,f_N}$ is a Bayes estimator and has therefore the minimal expected squared error amongst all estimators. This allows us to further rewrite $\epsilon^{2}_{\bq{}}$ into the following minimax forms:
\begin{align}
\epsilon^{2}_{\bq{}} &= \sup_{\substack{f\in\He\\\Hnorm{f}=1}} \left| \int f_x p(x) dx - \sum_{n=1}^{N}w^{(n)}_{\bq{}} f_{x_n}\right|^2\\
	&= \inf_{\hat{Z}:\mathcal{X}^N\mapsto\mathbb{R}} \sup_{\substack{f\in\He\\\Hnorm{f}=1}} \left| Z - \hat{Z}\left(f_{x_1},\ldots,f_{x_N}\right)\right|^2\\
	&= \inf_{\bm{w}\in\mathbb{R}^N} \sup_{\substack{f\in\He\\\Hnorm{f}=1}} \left| \int f_x p(x) dx - \sum_{n=1}^{N}w_n 	f_{x_n}\right|^2
\end{align}
Looking at $\epsilon^{2}_{\bq{}}$  this way, we may discover the deep similarity to the criterion $\epsilon^2_{herding}$ that kernel herding minimises. Optimal sampling for Bayesian quadrature minimises the same objective as kernel herding, but with the uniform $\frac{1}{N}$ weights replaced by the optimal weights. As a corollary
\begin{align}
\epsilon^{2}_{\bq{}}(x_1,\ldots,x_N)  \leq \epsilon^{2}_{KH} (x_1,\ldots,x_N)
\end{align}

It is interesting that $\epsilon^{2}_{\bq{}}$ has both a Bayesian interpretation as posterior variance under a Gaussian process prior, and a frequentist interpretation as a minimax bound on estimation error with respect to an RKHS.

\section{SUBMODULARITY}

\label{sec:submodularity}

In this section, we use the concept of approximate submodularity \citep{KrauseCevher10}, in order to study convergence properties of \sbq{}.

A set function $s:2^\mathcal{X} \mapsto \mathbb{R}$ is \textit{submodular} if, for all $A\subseteq B\subseteq \mathcal{X}$ and $\forall x \in \mathcal{X}$
\begin{align}
s(A\cup\{x\})-s(A)\geq s(B\cup\{x\})-s(B)
\end{align}
Intuitively, submodularity is a diminishing returns property: adding an element to a smaller set has larger relative effect than adding it to a larger set. A key result \cite[see e.\,g.\ ][and references therein]{KrauseCevher10} is that greedily maximising a submodular function is guaranteed not to differ from the optimal strategy by more than a constant factor of $(1-\frac{1}{e})$.

Herding and \sbq{} are examples of greedy algorithms optimising set functions: they add each pseudosamples in such a way as to minimize the instantaneous reduction in $\mmd$. So it is intuitive to check whether the objective functions these methods minimise are submodular. Unfortunately, neither $\epsilon_{herding}$, nor $\epsilon_{\bq{}}$ satisfies all conditions necessary for submodularity. However, noting that \sbq{} is identical to the sparse dictionary selection problem studied in detail by \citet{KrauseCevher10}, we can conclude that \sbq{} satisfies a weaker condition called \emph{approximate submodularity}. 

A set function $s:2^\mathcal{X} \mapsto \mathbb{R}$ is \textit{approximately submodular} with constant $\epsilon>0$, if for all $A\subseteq B\subseteq \mathcal{X}$ and $\forall x \in \mathcal{X}$
\begin{align}
s(A\cup\{x\})-s(A)\geq s(B\cup\{x\})-s(B) - \epsilon
\end{align}

\begin{prop}\label{prop:submodularity_SBQ}
$\epsilon^{2}_{\bq{}}(\emptyset)-\epsilon^{2}_{\bq{}}(\cdot)$ is a weakly submodular set function with constant $\epsilon<4r$, where $r$ is the incoherency
\begin{equation}
	r = \max_{x,x'\in\mathcal{P}\subseteq\mathcal{X}} \frac{k(x,x')}{\sqrt{k(x,x)k(x',x')}}
\end{equation}
\end{prop}
\begin{proof} By the definition of $\mmd$ we can see that
$-\epsilon^{2}_{\bq{}} = -\inf_{w\in\mathbb{R}^N}\Hnorm{\mu_p - \sum_{n=1}^N w^{(n)}_{\bq{}}k(\cdot,\vx_n)}^2$ is the negative squared distance between the mean element $\mu_p$ and its projection onto the subspace spanned by the elements $k(\cdot,\vx_n)$. Substituting $k=1$ into Theorem 1 of \citet{KrauseCevher10} concludes the proof.
\end{proof}

Unfortunately, weak submodularity does not provide the strong near-optimality guarantees as submodularity does . If $s:2^\mathcal{X} \mapsto \mathbb{R}$ is a weakly submodular function with constant $\epsilon$, and $\vert\mathcal{A}_n\vert=n$ is the result of greedy optimisation of $s$, then
\begin{equation}
	s(\mathcal{A}_n) \geq \left(1-\frac{1}{e}\right)\max_{\vert\mathcal{A}\vert\leq n}s(\mathcal{A}) - n\epsilon
\end{equation}

As pointed out by \citet{KrauseCevher10}, this guarantee is very weak, as in our case the objective function $\epsilon^{2}_{\bq{}}(\emptyset)-\epsilon^{2}_{\bq{}}(\cdot)$ is upper bounded by a constant. However, establishing a connection between \sbq{} and sparse dictionary selection problem opens up interesting directions for future research, and it may be possible to apply algorithms and theory developed for sparse dictionary selection to kernel-based quasi-Monte Carlo methods.

\section{EXPERIMENTS}
\label{sec:experiments}

In this section, we examine empirically the rates of convergence of sequential Bayesian quadrature and herding.  We examine both the expected error rates, and the empirical error rates.

In all experiments, the target distribution $p$ is chosen a 2D mixture of 20 Gaussians, whose equiprobability contours are shown in Figure \ref{fig:fig1}. To ensure a comparison fair to herding, the target distribution, and the kernel used by both methods, correspond exactly to the one used in \citep[Fig. 1]{chen2010super}. 
For experimental simplicity, each of the sequential sampling algorithms minimizes the next sample location from a pool of 10000 locations randomly drawn from the base distribution. In practice, one would run a local optimizer from each of these candidate locations, however in our experiments we found that this did not make a significant difference in the sample locations chosen. 

\subsection{Matching a distribution}

We first extend an experiment from \citep{chen2010super} designed to illustrate the mode-seeking behavior of herding in comparison to random samples. In that experiment, 
it is shown that a small number of i.\,i.\,d.\ samples drawn from a multimodal distribution will tend to, by chance, assign too many samples to some modes, and too few to some other modes. In contrast, herding places `super-samples' in such a way as to avoid regions already well-represented, and seeks modes that are under-represented.

We demonstrate that although herding improves upon i.\,i.\,d.\ sampling, the uniform weighting of super-samples leads to sub-optimal performance.  Figure \ref{fig:fig1} shows the first 20 samples chosen by kernel herding, in comparison with the first 8 samples chosen by \sbq{}.  By weighting the 8 \sbq{} samples by the quadrature weights in \eqref{eq:bq_weights}, we can obtain the same expected loss as by using the 20 uniformly-weighted herding samples.  
\begin{figure}
\includegraphics[width=\columnwidth]{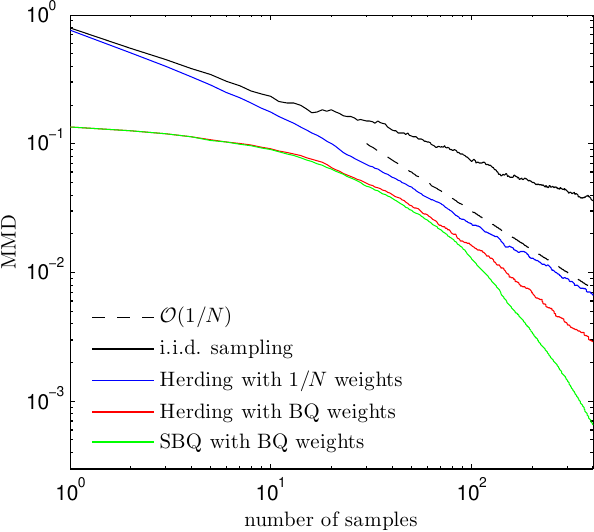}
\caption{The maximum mean discrepancy, or expected error of several different quadrature methods.  Herding appears to approach a rate close to $\mathcal{O}(1/N)$.  \sbq{} appears to attain a faster, but unknown rate.}
\label{fig:mmd_curve}
\end{figure}
Figure \ref{fig:mmd_curve} shows MMD versus the number of samples added, on the distribution shown in Figure \ref{fig:fig1}.  We can see that in all cases, \sbq{} dominates herding.  It appears that \sbq{} converges at a faster rate than $\mathcal{O}(1/N)$, although the form of this rate is unknown.

There are two differences between herding and \sbq{}:  \sbq{} chooses samples according to a different criterion, and also weights those samples differently.  We may ask whether the sample locations or the weights are contributing more to the faster convergence of \sbq{}. Indeed, in Figure \ref{fig:fig1} we observe that the samples selected by \sbq{} are quite similar to the samples selected by kernel herding. To answer this question, we also plot in Figure \ref{fig:mmd_curve} the performance of a fourth method, which selects samples using herding, but later re-weights the herding samples with \bq{} weights.  Initially, this method attains similar performance to \sbq{}, but as the number of samples increases, \sbq{} attains a better rate of convergence.  This result indicates that the different sample locations chosen by \sbq{}, and not only the optimal weights, are responsible for the increased convergence rate of \sbq{}.

\subsection{Estimating Integrals}


We then examined the empirical performance of the different estimators at estimating integrals of real functions.  To begin with, we looked at performance on 100 randomly drawn functions, of the form:
\begin{align}
f(\vx) & = \sum_{i=1}^{10} \alpha_i k(\vx, \vc_i)
\end{align}
where
\begin{align}
\Hnorm{f}^2 = \sum_{i=1}^{10} \sum_{j=1}^{10} \alpha_i \alpha_j k(\vc_i, \vc_j) = 1
\end{align}
That is, these functions belonged exactly to the unit ball of the RKHS defined by the kernel $k(\vx, \vx')$ used to model them.
\begin{figure}
\includegraphics[width=\columnwidth]{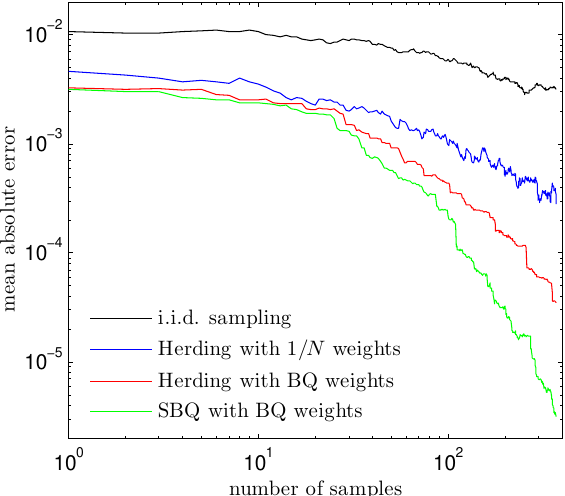}
\caption{Within-model error: The empirical error rate in estimating $Z_{f,p}$, for several different sampling methods, averaged over 250 functions randomly drawn from the RKHS corresponding to the kernel used.}
\label{fig:error_curve}
\end{figure}
Figure \ref{fig:error_curve} shows the empirical error versus the number of samples, on the distribution shown in Figure \ref{fig:fig1}.  The empirical rates attained by the method appear to be similar to the MMD rates in Figure \ref{fig:mmd_curve}.

By definition, MMD provides a upper bound on the estimation error in the integral of any function in the unit ball of the RKHS (Eqn.\ \eqref{eqn:rkhs-mmd}), including the Bayesian estimator, \sbq{}. Figure \ref{fig:bound_curve} demonstrates this quickly decreasing bound on the \sbq{} empirical error.

\begin{figure}
\includegraphics[width=\columnwidth]{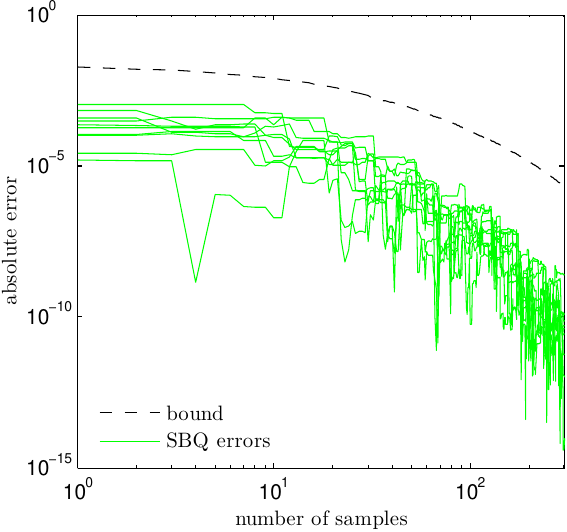}
\caption{The empirical error rate in estimating $Z_{f,p}$,  for the \sbq{} estimator, on 10 random functions drawn from the RKHS corresponding to the kernel used.  Also shown is the upper bound on the error rate implied by the $\mmd$.}
\label{fig:bound_curve}
\end{figure}

\subsection{Out-of-model performance}

A central assumption underlying \sbq{} is that the integrand function belongs to the RKHS specified by the kernel.  To see how performance is effected if this assumption is violated, we performed empirical tests with functions chosen from outside the RKHS.  We drew 100 functions of the form:
\begin{align}
f(\vx) & = \sum_{i=1}^{10} \alpha_i \exp(-\frac{1}{2} (\vx -\vc_i)^T \Sigma_i^{-1} (\vx -\vc_i)
\end{align}
where each $\alpha_i$ $\vc_i$ $\Sigma_i$ were drawn from broad distributions.  This ensured that the drawn functions had features such as narrow bumps and ridges which would not be well modelled by functions belonging to the isotropic kernel defined by $k$.
\begin{figure}
\includegraphics[width=\columnwidth]{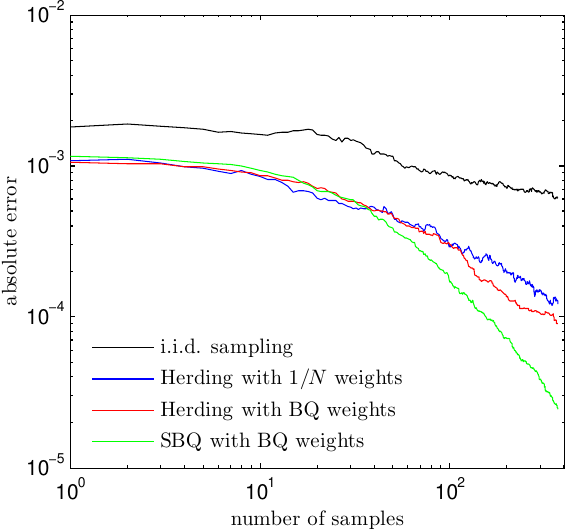}
\caption{Out-of-model error: The empirical error rates in estimating $Z_{f,p}$, for several different sampling methods, averaged over 250 functions drawn from outside the RKHS corresponding to the kernels used.}
\label{fig:error_curve_outmodel}
\end{figure}
Figure \ref{fig:error_curve_outmodel} shows that, on functions drawn from outside the assumed RKHS, relative performance of all methods remains similar.

Code to reproduce all results is available at \texttt{github.com/duvenaud/herding-paper}

\section{DISCUSSION}

\subsection{Choice of Kernel}

Using herding techniques, we are able to achieve fast convergence on a Hilbert space of \emph{well-behaved} functions, but this fast convergence is at the expense of the estimate not necessarily converging for functions outside this space.
If we use a characteristic kernel \citep{Sriperumbudur2010}, such as the exponentiated-quadratic or Laplacian kernels, then convergence in MMD implies weak convergence of $q_N$ to the target distribution. 
This means that the estimate converges for any bounded measurable function $f$. The speed of convergence, however, may not be as fast.

Therefore it is crucial that the kernel we choose is representative of the function or functions $f$ we will integrate.  For example, in our experiments, the convergence of herding was sensitive to the width of the Gaussian kernel.  One of the major weaknesses of kernel methods in general is the difficulty of setting kernel parameters.  A key benefit of the Bayesian interpretation of herding and MMD presented in this paper is that it provides a recipe for adapting the Hilbert space to the observations $f(x_n)$.  To be precise, we can fit the kernel parameters by maximizing the marginal likelihood of Gaussian process conditioned on the observations.  Details can be found in \citep{rasmussen38gaussian}.

\subsection{Computational Complexity}

While we have shown that Bayesian Quadrature provides the optimal re-weighting of samples, computing the optimal weights comes at an increased computational cost relative to herding. 
The computational complexity of computing Bayesian quadrature weights for $N$ samples is $\mathcal{O}(N^3)$, due to the necessity of inverting the Gram matrix $K(\vx, \vx)$.  Using the Woodbury identity, the cost of adding a new sample to an existing set is $\mathcal{O}(N^2)$.  For herding, the computational complexity of evaluating a new sample is only $\mathcal{O}(N)$, making the cost of choosing $N$ herding samples $\mathcal{O}(N^2)$.  For Monte Carlo sampling, the cost of adding an i.i.d. sample from the target distribution is only $\mathcal{O}(1)$.

\begin{table*}[t]
\begin{center}
\begin{tabular}{c|ccc}
method & complexity & rate & guarantee\\
\midrule
MCMC & $\mathcal{O}(N)$ & variable & ergodic theorem\\
i.i.d. MC & $\mathcal{O}(N)$ & $\frac{1}{\sqrt{N}}$ & law of large numbers\\
herding & $\mathcal{O}(N^2)$ & $\frac{1}{\sqrt{N}} \geq \cdot \geq \frac{1}{N}$ & \citep{chen2010super,bach2012equivalence} \\
SBQ & $\mathcal{O}(N^3)$ & unknown & approximate submodularity\\
\end{tabular}
\end{center}
\caption{A comparison of the rates of convergence and computational complexity of several integration methods.}
\label{tbl:rates}
\end{table*}

The relative computational cost of computing samples and weights using \bq{}, herding, and sampling must be weighed against the cost of evaluating $f$ at the sample locations.  Depending on this trade-off, the three sampling methods form a Pareto frontier over computational speed and estimator accuracy.  When computing $f$ is cheap, we may wish to use Monte Carlo methods.  In cases where $f$ is computationally costly, we would expect to choose the \sbq{} method.  When $f$ is relatively expensive, but a very large number of samples are required, we may choose to use kernel herding instead.  However, because the rate of convergence of \sbq{} is faster, there may be situations in which the $\mathcal{O}(N^3)$ cost is relatively inexpensive, due to the smaller $N$ required by \sbq{} to achieve the same accuracy as compared to using other methods.  

There also exists the possibility to switch to a less costly sampling algorithm as the number of samples increases.
Table \ref{tbl:rates} summarizes the rates of convergence of all the methods considered here.

\section{CONCLUSIONS}

In this paper, we have shown three main results:  First, we proved that the loss minimized by kernel herding is closely related to the loss minimized by Bayesian quadrature, when selecting sample locations. This implies that sequential Bayesian quadrature can viewed as an optimally-weighted version of kernel herding.

Second, we showed that the loss minimized by the Bayesian method is approximately submodular with respect to the samples chosen, and established connections to the submodular dictionary selection problem studied in \citep{KrauseCevher10}.

Finally, we empirically demonstrated a superior rate of convergence of \sbq{} over herding, and demonstrated a bound on the empirical error of the Bayesian quadrature estimate.

\subsection{Future Work}

In section \ref{sec:submodularity}, we showed that \sbq{} is approximately submodular, which provides only weak sub-optimality guarantees of its performance. It would be of interest to further explore the connection between Bayesian Quadrature and the dictionary selection problem to see if algorithms developed for dictionary selection can provide further practical or theoretical developments. The results in section \ref{sec:experiments}, specifically Figure \ref{fig:mmd_curve}, suggest that the convergence rate of \sbq{} is faster than $\mathcal{O}(1/N)$. However, we are not aware of any work showing what the theoretically optimal rate is. It would be of great interest to determine this optimal rate of convergence for particular classes of kernels.

\subsection*{Acknowledgements}

The authors would like to thank Carl Rasmussen and Francis Bach for helpful discussions, and Yutian Chen for his help in reproducing experiments.
We also thank Simon Lacoste-Julien for many helpful comments.

\bibliographystyle{icml2012}
\bibliography{herding}

\begin{thebibliography}{12}
\providecommand{\natexlab}[1]{#1}
\providecommand{\url}[1]{\texttt{#1}}
\expandafter\ifx\csname urlstyle\endcsname\relax
  \providecommand{\doi}[1]{doi: #1}\else
  \providecommand{\doi}{doi: \begingroup \urlstyle{rm}\Url}\fi

\bibitem[Bach et~al.(2012)Bach, Lacoste-Julien, and
  Obozinski]{bach2012equivalence}
Bach, F., Lacoste-Julien, S., and Obozinski, G.
\newblock On the equivalence between herding and conditional gradient
  algorithms.
\newblock \emph{Arxiv preprint arXiv:1203.4523}, 2012.

\bibitem[Chen et~al.(2010)Chen, Welling, and Smola]{chen2010super}
Chen, Y., Welling, M., and Smola, A.
\newblock Super-samples from kernel herding.
\newblock UAI, 2010.

\bibitem[Cowles \& Carlin(1996)Cowles and Carlin]{CowlesCarlin96}
Cowles, Mary~K. and Carlin, Bradley~P.
\newblock Markov chain monte carlo convergence diagnostics: A comparative
  review.
\newblock \emph{Journal of the American Statistical Association}, 91\penalty0
  (434):\penalty0 883--904, 1996.
\newblock ISSN 01621459.
\newblock \doi{10.2307/2291683}.
\newblock URL \url{http://dx.doi.org/10.2307/2291683}.

\bibitem[Krause et~al.(2006)Krause, Guestrin, Gupta, and Kleinberg]{guestrin1}
Krause, A., Guestrin, C., Gupta, A., and Kleinberg, J.
\newblock Near-optimal sensor placements: Maximizing information while
  minimizing communication cost.
\newblock In \emph{Proceedings of the Fifth International Conference on
  Information Processing in Sensor Networks (IPSN '06)}, pp.\  2--10,
  Nashville, Tennessee, USA, 2006.

\bibitem[Krause \& Cevher(2010)Krause and Cevher]{KrauseCevher10}
Krause, Andreas and Cevher, Volkan.
\newblock Submodular dictionary selection for sparse representation.
\newblock In F{\"u}rnkranz, Johannes and Joachims, Thorsten (eds.),
  \emph{ICML}, pp.\  567--574. Omnipress, 2010.
\newblock ISBN 978-1-60558-907-7.

\bibitem[Minka(2000)]{minka2000dqr}
Minka, T.~P.
\newblock {Deriving quadrature rules from Gaussian processes}.
\newblock Technical report, Statistics Department, Carnegie Mellon University,
  2000.

\bibitem[O'Hagan(1991)]{BZHermiteQuadrature}
O'Hagan, A.
\newblock Bayes-{H}ermite quadrature.
\newblock \emph{Journal of Statistical Planning and Inference}, 29:\penalty0
  245--260, 1991.

\bibitem[Rasmussen \& Ghahramani(2003)Rasmussen and Ghahramani]{BZMonteCarlo}
Rasmussen, C.~E. and Ghahramani, Z.
\newblock {B}ayesian monte carlo.
\newblock In Becker, S. and Obermayer, K. (eds.), \emph{Advances in Neural
  Information Processing Systems}, volume~15. MIT Press, Cambridge, MA, 2003.

\bibitem[Rasmussen \& Williams(2006)Rasmussen and
  Williams]{rasmussen38gaussian}
Rasmussen, C.E. and Williams, CKI.
\newblock {Gaussian Processes for Machine Learning}.
\newblock \emph{The MIT Press, Cambridge, MA, USA}, 2006.

\bibitem[Song et~al.(2008)Song, Zhang, Smola, Gretton, and
  Sch\"{o}lkopf]{Song2008}
Song, Le, Zhang, Xinhua, Smola, Alex, Gretton, Arthur, and Sch\"{o}lkopf,
  Bernhard.
\newblock Tailoring density estimation via reproducing kernel moment matching.
\newblock In \emph{Proceedings of the 25th international conference on Machine
  learning}, ICML '08, pp.\  992--999, New York, NY, USA, 2008. ACM.
\newblock ISBN 978-1-60558-205-4.
\newblock \doi{10.1145/1390156.1390281}.
\newblock URL \url{http://doi.acm.org/10.1145/1390156.1390281}.

\bibitem[Sriperumbudur et~al.(2010)Sriperumbudur, Gretton, Fukumizu,
  Sch\"{o}lkopf, and Lanckriet]{Sriperumbudur2010}
Sriperumbudur, Bharath~K., Gretton, Arthur, Fukumizu, Kenji, Sch\"{o}lkopf,
  Bernhard, and Lanckriet, Gert~R.G.
\newblock Hilbert space embeddings and metrics on probability measures.
\newblock \emph{J. Mach. Learn. Res.}, 99:\penalty0 1517--1561, August 2010.
\newblock ISSN 1532-4435.
\newblock URL \url{http://dl.acm.org/citation.cfm?id=1859890.1859901}.

\bibitem[Welling(2009)]{welling2009herding}
Welling, M.
\newblock Herding dynamical weights to learn.
\newblock In \emph{Proceedings of the 26th Annual International Conference on
  Machine Learning}, pp.\  1121--1128. ACM, 2009.

\end{thebibliography}

\end{document}